\documentclass[twoside,11pt]{article}

\usepackage{jmlr2e}
 \usepackage[T1]{fontenc}
 \usepackage{tgtermes}
 \usepackage{commath}
 \usepackage[normalem]{ulem} 
 \usepackage[mathscr]{eucal} 
 \usepackage{mathtools}
 \usepackage{booktabs} 
 \usepackage{amsmath}
 \usepackage{amsfonts}
 \usepackage{times} 
 \usepackage{ulem} 
 \usepackage{xcolor} 
 \usepackage{listings} 
 \usepackage{tabularx}
 \usepackage{microtype}
 \usepackage[british]{babel} 
 \usepackage[autostyle]{csquotes} 
 \usepackage{enumitem}
 \usepackage{color}
 \usepackage{stmaryrd}
 \SetSymbolFont{stmry}{bold}{U}{stmry}{m}{n}
 \usepackage{caption}
 \usepackage{subfig}
 \newsavebox\myVerb

 \usepackage{mymacros} 
 \usepackage{tikz,tikz-3dplot}
 
 \tikzset{>=latex}
 \usepackage{pgfplots} 
 \pgfplotsset{compat=1.13}
 \usetikzlibrary{decorations.pathmorphing,decorations.pathreplacing,decorations.shapes}
 \usetikzlibrary{shapes,arrows}
  \usetikzlibrary{positioning,arrows.meta,quotes}
  \usetikzlibrary{bayesnet}
  \usepackage{algorithmic} 
  \usepackage[ruled,linesnumbered,titlenotnumbered]{algorithm2e}

\newcommand{\kk}[1]{\textcolor{black}{#1}}

\ShortHeadings{Efficient Structure-preserving STTM}{ Kour, Dolgov, Stoll and Benner}
\firstpageno{1}

\begin{document}

\title{Efficient Structure-preserving Support Tensor Train Machine }

\author{\name{Kirandeep Kour} \email{kour@mpi-magdeburg.mpg.de}\AND
     	\name{Peter {Benner}} \email{benner@mpi-magdeburg.mpg.de}\\
       \addr
       Computational Methods in Systems and Control Theory \\
       Max Planck Institute for Dynamics of Complex Technical Systems\\
       Magdeburg, D-39106, Germany. \\
       \AND
       \name Sergey Dolgov \email S.Dolgov@bath.ac.uk  \\
       \addr
       Department of Mathematical Sciences\\
       University of Bath\\
       Bath BA2 7AY, United Kingdom.\\
       \AND
       \name  Martin Stoll \email martin.stoll@mathematik.tu-chemnitz.de \\
       \addr
       Faculty of Mathematics\\
       Technische Universit\"{a}t Chemnitz\\
       Chemnitz, D-09107, Germany.\\
    }
\editor{}

\maketitle

\begin{abstract}
	An increasing amount of  collected data are high-dimensional multi-way arrays (tensors), and it is crucial for efficient learning algorithms to exploit this tensorial structure as much as possible. The ever present curse of dimensionality for high dimensional data and the loss of structure when vectorizing the data motivates the use of tailored low-rank tensor classification methods. In the presence of  small amounts of training data, kernel methods offer  an attractive choice as they provide the possibility for a nonlinear decision boundary.  \kk{We develop the Tensor Train Multi-way Multi-level
Kernel (TT-MMK), which combines the simplicity of the Canonical Polyadic  decomposition, the classification power of the Dual Structure-preserving Support Vector Machine, and the reliability of the Tensor Train (TT) approximation.}
We show by experiments that the TT-MMK method is usually more reliable computationally, less sensitive to tuning parameters, and gives higher prediction accuracy in the SVM classification when benchmarked against other state-of-the-art techniques.
\end{abstract}

\section{Introduction}

In many real world applications, data often emerges in the form of high-dimensional tensors.
It is typically very expensive to generate or collect such data, and we assume that we might be given a rather small amount of test and training data.
Nevertheless, it remains crucial to be able to classify tensorial data points.
A prototypical example of this type is fMRI brain images~\citep{GLOVER2011133}, which consist of three-dimensional tensors of voxels, and may also be equipped with an additional temporal dimension,
in contrast to traditional two-dimensional pixel images.

One of the most popular methods for classifying data points are Support Vector Machines (SVM)~\citep{vapnik,vapnik98}.
These are based on margin maximization and the computation of the corresponding weights via an optimization framework, typically the SMO algorithm~\citep{Platt98sequentialminimal}.
These methods often show outstanding performance, but the standard SVM model~\citep{vapnik95} is designed for vector-valued rather than tensor-valued data.
Although tensor objects can be reshaped into vectors, much of the information inherent in the tensorial data is lost.
For example, in an fMRI image, the values of adjacent voxels are often close to each other~\citep{DuSK}.
As a result, it was proposed to replace the vector-valued SVM by a tensor-valued SVM.
This area was called Supervised Tensor Learning (STL)~\citep{Tao07,Zhou13,Guo12}.
In~\citet{wolf}, the authors proposed to minimize the rank of the weight parameter with the orthogonality constraints on the columns of the weight parameter instead of the classical maximum-margin criterion, and~\citet{Pirisiavash} relaxed the orthogonality constraints to further improve the Wolf\textquotesingle s method.
~\citet{Hao13} consider a rank-one tensor factorization of each input tensor, while~\citet{kotsia} adopted the Tucker decomposition of the weight parameter instead of rank-one tensor decompositions to retain a more structural information.
~\citet{Zeng17} extended this by using a Genetic Algorithm (GA) prior to the Support Tucker Machine (STuM) for the contraction of the input feature tensor.
Along with these rank-one and Tucker representations, recently the weight tensor of  STL has been approximated using the Tensor Train (TT) decomposition~\citep{chen2018support}.
We point out that these methods are mainly focusing on a linear representation of the data.
It is well known that a linear decision boundary is often not suitable for the separation of complicated real world data~\citep{MLbook1}.

Naturally, the goal is to design a nonlinear transformation of the data, and we refer to~\citet{Signoretto11,Signoretto12,QZhao13a}, where kernel methods have been used for tensor data.
All these methods are based on the Multi-linear Singular Value Decomposition/Higher Order Singular Value Decomposition, which rely on the flattening of the tensor data.
Therefore, the resulting vector and matrix dimensions are so high that the methods are prone to over-fitting.
Moreover, the intrinsic tensor structure is typically lost. Thus, other approaches are desired.

The approximation of tensors based on low-rank decompositions has received a lot of attention in scientific computing over  recent years~\citep{cichocki2016tensor,Kolda09,Cichocki13,Liu15}.
\kk{A Dual Structure-preserving Kernel (DuSK) for STL, which is
particularly tailored to SVM and tensor data, was introduced in~\citep{DuSK}.
This kernel is defined on the Canonical Polyadic (CP) tensor format, also known as Parallel Factor Analysis, or PARAFAC~\citep{Hitchcock1927,Hitchcock1928}.}
\kk{Once the CP format is available, DuSK delivers an accurate and efficient classification, but the CP approximation of arbitrary data can be numerically unstable and difficult to compute~\citep{desilva2008}. In general, any optimization method (Newton, Steepest Descent or Alternating Least Squares) might return only a locally optimal solution, and it is difficult to assess whether this is a local or global optimum.}
Later on, kernelized tensor factorizations, specifically a Kernelized-CP (KCP) factorization, have been introduced in~\citet{MMK}, and the entire technique has been called the Multi-way Multi-level Kernel (MMK) method.
Further elaboration and understanding of the KCP approach~\citep{KSTM} is provided by a kernelized Tucker model, inspired by~\citet{Signoretto2013}.

Recently, kernel approximations in the TT format have been introduced in~\citet{chen2020kernelized}.
Initially, we had pursued a similar idea for fMRI data sets, but
we observed that the nonlinear SVM classification
using directly the TT factors leads to poor accuracy,
since different TT factors have different dimensions and scales,
making the feature space more complicated.
Hence, we have come up with a better exploitation of the data structure, as we explain in this paper.

Tensor decompositions and kernel-based methods have become an indispensable tool in many learning tasks.
For example,~\citet{exponential} uses the TT decomposition for both the input tensor and the corresponding weight parameter in generalized linear models in machine learning.
A Kernel Principal Component Analysis (KPCA), a kernel-based nonlinear feature extraction technique, was proposed in~\citet{Wu07}.
The authors of~\citet{speedingupCNN} propose a way to speed up Convolutional Neural Networks (CNN) by applying a low-rank CP decomposition on the kernel projection tensor.
\subsection{Main Novelty}
In this paper, we \kk{develop an efficient structure-preserving nonlinear kernel function for SVM classification of tensorial data.
We start with the TT approximation of the data points, which can be computed reliably by the TT-SVD algorithm.
Moreover, we enforce uniqueness of the SVD factors,
such that ``close'' tensors yield ``close'' TT factors.
Second, we perform an exact expansion of the TT decomposition into the CP format.
This unifies the dimensions of the data used in classification.
Finally, we redistribute the norms of the CP factors to equilibrate the actual scales of the data elements.
This yields  a CP decomposition that is free from scaling indeterminacy, while being a reliable approximation of the original data.
We have observed that using this decomposition in DuSK significantly increases the classification accuracy and stability of the STL.}

The paper is structured as follows.
In~\secref{preliminaries}, we set the stage introducing basic definitions and important tools.
An extension to the tensor format SVM is explained in~\secref{framektt},
where we also introduce the Kernelized Support Tensor Machine (KSTM) via the kernel trick (\secref{fmapkertrick}).
In~\secref{overalgo} we explain the entire proposed algorithm step by step.
In particular, we introduce the uniqueness enforcing TT-SVD algorithm (\secref{sec: USVD}), the TT-CP expansion (\secref{TT-CP}) and the norm equilibration (\secref{sec:normequ}).
In~\secref{numel} we benchmark the different steps of the proposed algorithm and compare it to a variety of competing methods using two data sets each from two different fields with a limited amount of training data, which are known to be challenging for classification.
\section{Preliminaries} \label{preliminaries}
This section introduces terminology and definitions used throughout the paper.
\subsection{Tensor Algebra}
A tensor  is a multidimensional array~\citep{Kolda09} which is a higher order generalization of vectors and matrices. We denote an $M^{th}$-order tensor ($M \geq 3$) by a calligraphic letter \tensx $\in$ \dims R I M, its entries by \tenem x, a matrix by a boldface upper case letter \bX $\in$ \RIJ, and a vector by a boldface lower case letter $\bx \in \RI$. Matrix and vector elements are denoted by $x_{ij} = \bX(i,j)$ and $x_i = \bx(i)$, respectively. The order of a tensor is the number of its \emph{dimensions}, \emph{ways} or \emph{modes}. The \emph{size} of a tensor stands for the maximum index value in each mode. For example, \tensx \  is of order $M$ and the size in each mode is $I_m$, where $  m \in \expect M:=\left\lbrace 1,2, \ldots , M \right\rbrace $. For simplicity, we assume that all tensors are real valued.
\begin{definition}{An $m$-mode matricization}
        \nten \tensx m $\in$ \mdims m M for $m \in  \expect M$
	is the unfolding (or flattening) of an $M^{th}$-order tensor into a matrix in the appropriate order of elements\kk{, \ie~ a tensor element $(i_1, i_2, \ldots i_M)$ maps to an element $(i_m,j)$ of a matrix as follows~\citep{Kolda09}:}
	\kk{$$ j = 1+ \sum_{k = 1, k \neq m}^{M} (i_k - 1) J_k ~~ \text{with} ~~ J_k= \prod_{\ell = 1, \ell \neq m}^{k-1} I_\ell.$$}
	 
\end{definition}
\begin{definition}{An $m$-mode product $\tensx \nmod m \bA  \  \in
	\dimnmod R m$}, given
	\tensx \ $\in$ \dims R I M and $\bA \in$ \matdim R, is defined as a tensor-matrix product in $m^\text{th}$ way:
	$$ \nten \tensy m  = \nten {(\tensx \nmod m  \bA)} m = \bA \nten \tensx m.$$
\end{definition}
\begin{definition}{A mode-($M$,1) contracted product $\tenz   = \tensx \times_{M}^{1} \tensy = \tensx \times^{1} \tensy \in$ \modnprod \R},
	for given tensors $\tensx \in \dims R I M$ and  $\tensy \in \dims R J M$, with $I_M = J_1$, yields a tensor $\tenz$ with entries  $$z_{i_1, \ldots, i_{M-1}, j_2, \ldots, j_M} = \sum_{i_M = 1}^{I_M} x_{i_1,\ldots,i_M} y_{i_M, j_2, \ldots, j_M}.$$
\end{definition}
\begin{definition}{The inner product}
	of given tensors \tensx, \tensy \ $\in$ \dims R I M is defined as
	$$\langle \tensx , \tensy \rangle = \teninner.$$
\end{definition}
\begin{definition}{The outer product}
	of given tensors $\tensx \in \dims R I M  \text{and} ~ \tensy \in \dims R J N $ generates an $\brac{M + N}^{th}-$ order tensor $\tenz = \tensx \circ \tensy$  with entries  $$z_{i_1, \ldots, i_{M}, j_1, \ldots, j_N} =  x_{i_1,\ldots,i_M} y_{j_1, \ldots, j_N}.$$
\end{definition}
\begin{definition}{The Kronecker Product}
	of matrices $\bA \in \RIJ, \bB \in \RKL$ is defined as usual by
	$$
	\bA \otimes \bB =  \begin{bmatrix}
	a_{1,1} \bB & \cdots & a_{1,J} \bB\\
	\vdots & \ddots & \vdots\\
	a_{I,1} \bB & \cdots & a_{I,J} \bB\\
	\end{bmatrix} \in \RIJKL.
	$$
	Similarly, the Kronecker product of two tensors $\tensx \in \dims R I M, \tensy \in \dims R J M$ returns a tensor $\tenz = \tensx \otimes \tensy \in \krodimten R I J$.
\end{definition}
Moreover, the Khatri-Rao product is a column-wise Kronecker product,
$$\bA \odot \bB = [\ba_1  \otimes \bb_1, \ba_2 \otimes \bb_2, \cdots, \ba_R \otimes \bb_R ] \in  \RIKR.$$
These notations are summarized in Table~\ref{notation}.
\begin{table}[ht]
	\caption{Tensor Notations.}
	\label{notation}
	\vskip 0.15in
	\begin{center}
		\begin{small}
				\begin{tabular}{lcr}
					\toprule
					Symbol & Description\\
					\midrule
					$x$    & Lower case letter for scalar value\\
					\bx  & Lower case bold letter for vector\\
					\bX    & Upper case bold letter for matrix\\
					\tensx    & Calligraphic bold letter for tensor\\
					\nten \tensx m    & Calligraphic bold letter with subscript m \\
					& for $m$-mode matricization\\
					$\circ$   & Outer product\\             
					$\otimes$      & Kronecker product\\
					$\odot$      & Khatri-Rao product\\
					$\times^{1}_{M}$ & Mode-$(M,1)$ contracted product\\
					$\langle M \rangle$ & Integer values from 1 to $M$\\
					$\langle \tensx, \tensy \rangle$      & Inner product for tensors \tensx \ and \tensy \\
					\bottomrule
				\end{tabular}
		\end{small}
	\end{center}
	\vskip -0.1in
\end{table}
\subsection{Tensor Decompositions}
Tensor decomposition methods have been significantly enhanced during the last two decades, and applied
to solve problems of varying computational complexity.
The main goal is the linear (or at most polynomial) scaling of the computational complexity in the dimension (order) of a tensor.
The key ingredient is the separation of variables via approximate low-rank factorizations.
In this paper we consider two of these decompositions.
\subsubsection{Canonical Polyadic decomposition}
The Canonical Polyadic (CP) decomposition of an $M^{th}-$order tensor $\tensx \in \dims R I M$ is a factorization into a sum of rank-one components~\citep{Hitchcock1927}, which is given element-wise as
\begin{align}\label{cpkrush}
\tenem x & \cong \sum_{r=1}^{R} \ba^{(1)}_{i_1,r} \ba^{(2)}_{i_2,r} \cdots \ba^{(M)}_{i_M,r},  \nonumber\\
 \mbox{or shortly,} ~~\qquad \quad \quad \tensx & \cong \llbracket \cpd \rrbracket,
\end{align}
where  $\mathbf{A}^{(m)} = \left[ \ba^{(m)}_{i_m,r} \right] \in \R^{I_m \times R}$, $m=1,\ldots,M$, are called \textit{factor matrices} of the CP decomposition, see~\figref{figcp}, and $R$ is called the CP-rank. The notation $\llbracket \cpd \rrbracket$ is also called the Kruskal representation of the CP factorization.
Despite the simplicity of the CP format, the problem of the
best CP approximation is often ill-posed~\citep{desilva2008}.
A practical CP approximation can be computed via the Alternating Least Squares (ALS) method~\citep{cp_als}, but the convergence may be slow.
It may also be difficult to choose the rank $R$.
\begin{figure}[ht]
	\begin{center}
		\includegraphics[scale = 1.12]{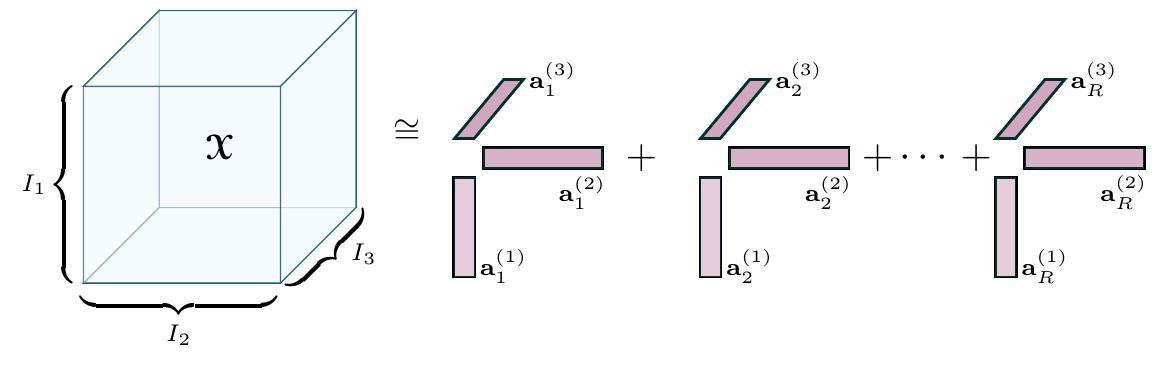}
		\caption{CP decomposition of a 3-way tensor.}
		\label{figcp}
	\end{center}
	\vskip -0.2in
\end{figure}
\subsubsection{Tensor Train decomposition}
To alleviate the difficulties of the CP decomposition mentioned above, we build
our proposed algorithm on
the Tensor Train (TT)~\citep{oseledets2011tensor} decomposition.
The TT approximation of an $M^{th}-$order tensor  $\tensx \in \dims R I M$ is defined element-wise as
\begin{align}\label{eq:tt}
\tenem x & \cong \sum_{r_0,\ldots,r_{M}}\tendtt G 1_{r_0,i_1,r_1} \tendtt G 2_{r_1,i_2,r_2} \cdots \tendtt G M_{r_{M-1},i_M,r_M}, \nonumber\\
\tensx & \cong\llangle \tendtt G 1, \tendtt G 2, \ldots, \tendtt G M \rrangle,
\end{align}
where \tendtt G m $\in$ \ttcdim m, $m=1,\ldots,M,$ are 3rd-order tensors called \emph{TT-cores} (see~\figref{figtt}), and $R_0,\ldots,R_M$ with $R_0 = R_M = 1$ are called \emph{TT-ranks}.
\begin{figure}[ht]
	\begin{center}
		\includegraphics[scale = 0.05]{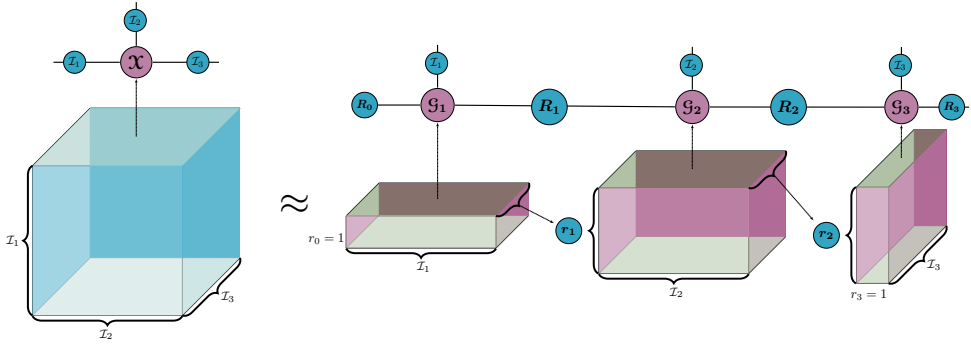}
		\caption{TT decomposition of a 3-way tensor.}
		\label{figtt}
	\end{center}   
	\vskip -0.2in
\end{figure}
The alluring capability of the TT format is its ability to perform algebraic operations directly on TT-cores avoiding full tensors.
Moreover, we can compute a quasi-optimal TT approximation of any given tensor using the SVD.
This builds on the fact that the TT decomposition constitutes a recursive matrix factorization, where each TT-rank is the matrix rank of the appropriate unfolding of the tensor, and hence the TT approximation problem is well-posed~\citep{oseledets2011tensor}.
\subsection{Support Vector Machine}\label{svm}
In this section, we recall the SVM method.
For a given training data set $\lbrace \left( \bx_i, y_i\right) \rbrace _ {i = 1}^N$, with \emph{input data} $\bx_i \in \R^m$ and \emph{labels} $y_i \in \lbrace-1,1\rbrace$,  the dual-optimization problem for the \textit{nonlinear} binary classification can be defined as,
\begin{align}\label{eq: dualPhiOpt}
&\underset{\alpha_1,\ldots,\alpha_N}{\text{max}}  \quad \sum_{i = 1}^{N} \balpha_i - \frac{1}{2} \sum_{i = 1}^{N}\sum_{j = 1}^{N} \balpha_i \balpha_j y_i y_j \langle \phi {(\bx_i)},\phi {(\bx_j)} \rangle \nonumber \\
&  \text{subject to}   \quad  0 \leq \balpha_i \leq C, \quad  \sum_{i = 1}^{N} \balpha_i y_i = 0,
\end{align}
where a tuning function $\phi$ defines the nonlinear decision boundary with $\phi\colon  \bx_i \rightarrow \phi\left( \bx_i\right)$.
In practice, we compute directly
$\langle \phi \brac{\bx_i},\phi \brac{\bx_j} \rangle$ using the so-called \emph{Kernel Trick}~\citep{bernhard2001}.
\subsubsection{Feature Map and Kernel Trick}\label{fmapkertrick}
The function $\phi \colon \R^m \rightarrow \mathbb{F}$ is called
\textbf{feature map}, and the \textit{feature space} $\mathbb{F}$ is a Hilbert Space (HS).
Every feature map is defined via a kernel such that $\nk_{i,j}=\nk \brac {\bx_i, \bx_j} = \langle \phi(\bx_i), \phi(\bx_j) \rangle_{\mathbb{F}} $. Employing the properties of the inner product, we conclude that $[\nk_{i,j}]$ is a symmetric positive semi-definite matrix.
The \emph{kernel trick} lies in
defining and computing directly $\nk\brac{\bx_i,\bx_j}$ instead of $\phi(\bx)$.
It is used to get a linear learning algorithm to learn a  \emph{nonlinear boundary}, without explicitly knowing the nonlinear function $\phi$.
The only task needed for the SVM is thus to choose a legitimate kernel function. That is how we work with the input data in the high-dimensional space while  doing all the computation in the original low dimensional space.~\figref{kertrick} illustrates the linear separation in a higher dimensional space.
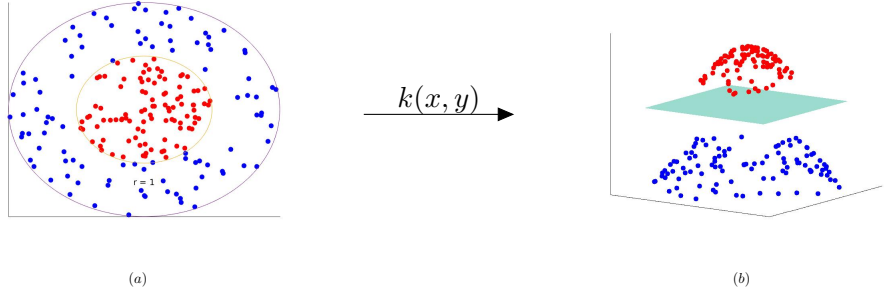
\begin{figure}[ht]
	\begin{center}
		\centerline{
			\begin{tikzpicture}
			\node at (-4,0) {\includegraphics[scale= 0.15]{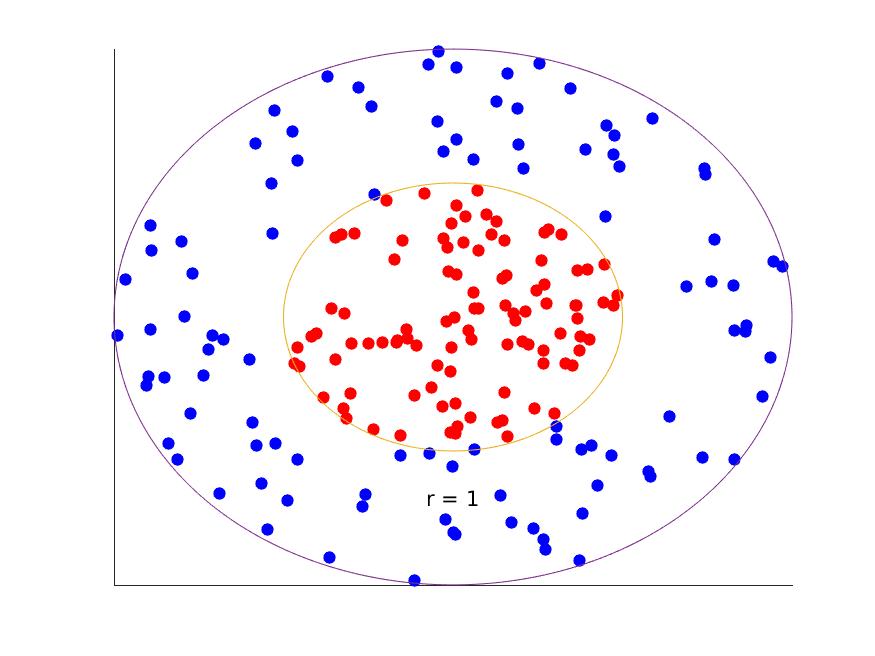}};
			\node at (4,0) {\includegraphics[scale=0.15]{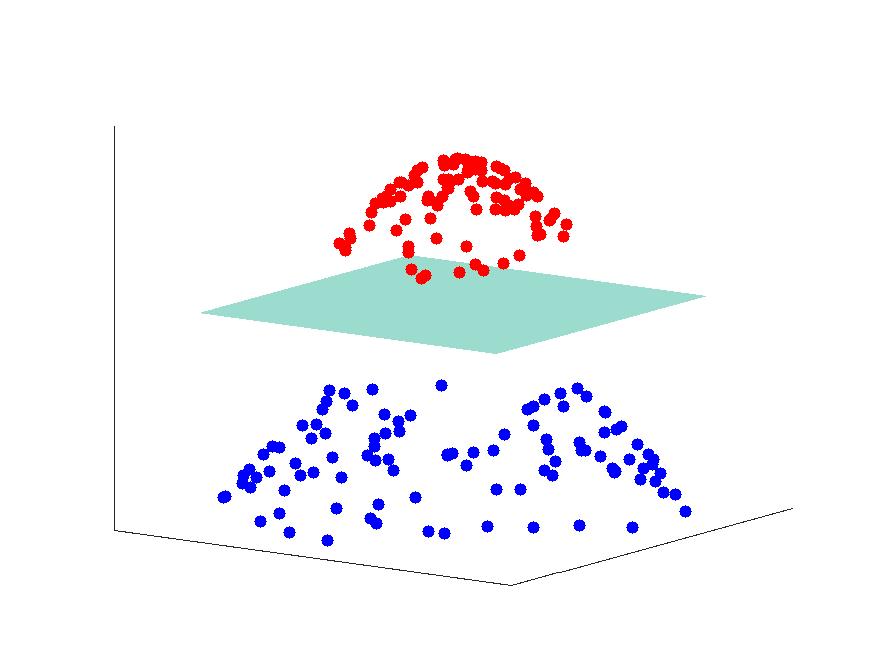}};
			\node at (0, 0.2) {$\nk(x,y)$};
			\node at (-4,-2.2) {\scalebox{0.5}{\brac{a}}};
			\node at (4,-2.2) {\scalebox{0.5}{ \brac{b}}};
			\draw[->]  (-1,0) -- (1,0);
			\end{tikzpicture}
		}
		\caption{Nonlinear mapping using kernel trick: \brac{a} Nonlinear classification of data in $\R^2$, \brac{b} Linear classification in higher dimension ($\R^3$).}
		\label{kertrick}
	\end{center}
	\vskip -0.2in
\end{figure}
\subsection{Kernelized Support Tensor Machine}\label{framektt}
In our case, we have a data set $\lbrace \left( \tensx_i, y_i\right) \rbrace _ {i = 1}^N$ with input data in the form of a tensor $\tensx_i \in \dims R I M$.
\kk{ We take the maximum margin approach to get the separation hyperplane. Hence, the objective function for a nonlinear boundary in the tensor space can be written as follows~\citep{Cai06supporttensor}:}
\kk{\begin{align}\label {eq: tenPrimOpt}
	&\underset{w,b}{\text{min}} \quad \frac{1}{2} \norm{w}^2 + C \sum_{i=1}^{N} \xi_i \\  \nonumber
	  \text{subject to} \quad & y_i (\langle \Psi(\tensx_{i}), w \rangle + b) \geq 1-\xi_{i} ~~ \quad \xi_{i}~\geq 0 ~~~ \forall i.	
\end{align}}
\kk{The classification setup given in~\eqref{eq: tenPrimOpt} is known as Support Tensor Machine (STM)~\citep{Tao05}. The dual formulation of the corresponding primal problem can be given as follows:}
\begin{align}\label {eq: tenDualOpt}
&\underset{\alpha_1,\ldots,\alpha_{N}}{\text{max}} \quad  \sum_{i = 1}^{N} \balpha_i - \frac{1}{2} \sum_{i = 1}^{N}\sum_{j = 1}^{N} \balpha_i \balpha_j y_i y_j \langle \Psi(\tensx_i),\Psi(\tensx_j) \rangle \nonumber\\
&  \text{subject to} \quad 0 \leq \balpha_i \leq C,  \quad \sum_{i = 1}^{N} \balpha_i y_i = 0 ~~ \forall i. 
\end{align}
The nonlinear feature mapping $\Psi \colon \dims R I M \rightarrow \mathbb{F}$ takes tensorial input data to a higher dimensional space similarly to the vector case. Therefore, by using the kernel trick, explained in~\secref{fmapkertrick}, STM can be defined as follows:
\begin{align} \label{eq: kerSTM}
&\underset{\alpha_1,\ldots,\alpha_{N}}{\text{max}} \quad  \sum_{i = 1}^{N} \balpha_i - \frac{1}{2} \sum_{i = 1}^{N}\sum_{j = 1}^{N} \balpha_i \balpha_j y_i y_j K(\tensx_i,\tensx_j) \nonumber \\
&  \text{subject to} \quad 0 \leq \balpha_i \leq C, \quad  \sum_{i = 1}^{N} \balpha_i y_i = 0 ~~ \forall i.
\end{align}
We call this setup the \emph{Kernelized STM (KSTM)}. \kk{Once we have the real-valued function (kernel) value for each pair of tensors, we can use state-of-the-art LIBSVM~\citep{libsvm}, which relies on the \emph{Sequential Minimal Optimization} algorithm to optimize the weights $\alpha_i$.} Hence, the preeminent part is the kernel function $K(\tensx_i,\tensx_j)$.
\kk{However, the direct treatment of large tensors can be both numerically expensive and inaccurate due to overfitting.
Therefore, we need to choose a kernel that exploits the tensor decomposition.}
In the next section we propose a particular choice of the kernel for tensor data.
\section{The Proposed Algorithm}\label{overalgo}

The first essential step towards using tensors is to approximate them in a low-parametric representation.
To achieve a stable learning model, we start with computing the TT approximations of all data tensors.
The second most expensive part is the computation of $K\brac{\tensx_{i}, \tensx_{j}}$ for each pair of tensors.
Therefore, an approximation of the kernel is required.
Besides, we would like the kernel to exploit the factorized tensor representation.
These issues are resolved in the rest of this section.
\subsection{Uniqueness Enforcing TT- SVD }\label{sec: USVD}
Since the TT decomposition is computed using the SVD~\citep{oseledets2011tensor}, the particular factors $\teng^{(1)},\teng^{(2)},\ldots, \teng^{(M)}$ are defined only up to a sign indeterminacy.
For example, in the first step, we compute the SVD of the $1$-mode matricization,
\[
\tensx_{(1)}  = \sigma_1 u_1 v_1^\top + \cdots + \sigma_{I_1} u_{I_1} v_{I_1}^\top,
\]
followed by truncating the expansion at rank $R_1$ or according to the accuracy threshold $\varepsilon$, choosing $R_1$ such that $\sigma_{R_1+1}<\varepsilon$.
However, any pair of vectors $\{u_{r_1}, v_{r_1}\}$ can be replaced by $\{-u_{r_1}, -v_{r_1}\}$ without changing the whole expansion.
While this is not an issue for data compression,
classification using TT factors can be affected significantly by this indeterminacy.
For example, tensors that are close to each other should likely produce the same label.
In contrast, even a small difference in the original data may lead to a different sign of the singular vectors.
\kk{Such a qualitative difference in the data which is actually used for classification
may significantly inflate the dimension of the feature space $\mathbb{F}$ in~\eqref{eq: tenDualOpt} required for the accurate separation of the classes.}

\kk{We fix the signs of the singular vectors as follows.
For each $r_1=1,\ldots,R_1$, we find the position of the maximum in modulus element in the left singular vector, $i_{r_1}^* = \arg\max_{i=1,\ldots,I_1} |u_{i,r_1}|$,
and make this element positive,}
\[
\kk{\bar u_{r_1}:=u_{r_1} / \mathrm{sign}(u_{i^*_{r_1},r_1}), \quad \bar v_{r_1}:= v_{r_1} \cdot \mathrm{sign}(u_{i^*_{r_1},r_1}).}
\]

\kk{Finally, we collect $\bar u_{r_1}$ into the first TT core, $\teng^{(1)}_{r_0,i_1,r_1} = \bar u_{i_1,r_1}$, and continue with the TT-SVD algorithm using $\bar v_{r_1}$ as the right singular vectors.
In contrast to the sign, the whole dominant singular terms $u_{r_1}v_{r_1}^\top$ depend continuously on the input data, and so do the maximum absolute elements. The procedure is summarized in Algorithm~\ref{UETTsvd}.}

\begin{algorithm}[ht]
	\caption{\kk{Uniqueness Enforcing TT-SVD}}
	\label{UETTsvd}
	\begin{algorithmic}[1]
		\small		
		\STATE {\bfseries Input:} $M$-dimensional tensor $\tensx \in \dims R I M$, relative error threshold $\epsilon$.
		\STATE {\bfseries Ensure:} Cores $\ttbra$ of the TT-approximation $\tensx'$ to $\tensx$ in the TT-format with
		TT- rounding ranks $r_m$ equal to the $\delta$-ranks of the unfoldings $\tensx_{(m)}$ of $\tensx$, where $\delta = \sqrt{\frac{\epsilon}{M-1}} \norm{\bA}_F$.
		\STATE Initialize $\hat \bZ_1 = \tensx_{(1)}, R_0 = 1$.
		\FOR { $m = 1$ to $M-1$}
		\STATE $\bZ_m : = \text{reshape} \left(\hat \bZ_m, [R_{m-1} I_m,~I_{m+1}\cdots I_M]\right)$
		\STATE Compute $\delta$-truncated SVD: $\bZ_m = \bU_m \bS_m \bV_m^T + \mathbf{E}_m, ~\norm{\mathbf{E}_m}_F \leq \delta$, where \\
		$\bU_m = [u_1^{(m)}, u_2^{(m)}, \ldots, u_{R_m}^{(m)}],
		\bS_m = \text{diag}(\sigma_1^{(m)}, \sigma_2^{(m)}, \ldots, \sigma_{R_m}^{(m)}), \bV_m = [v_1^{(m)}, v_2^{(m)}, \ldots, v_{R_m}^{(m)}]$
		\FOR {$r_m = 1$ to $R_m$}
		\STATE $i_{r_m}^* = \arg\max_{i=1,\ldots,R_{m-1}I_m} |u_{i,r_m}^{(m)}|$ (with ties broken to first element) \label{abs}
		\STATE $\bar u_{r_m}^{(m)}:=u_{r_m}^{(m)} / \mathrm{sign}(u_{i_{r_m}^*,r_m}^{(m)}), \quad \bar v_{r_m}^{(m)}:= v_{r_m}^{(m)} \cdot \mathrm{sign}(u_{i^*_{r_m},r_m}^{(m)})$
		\STATE $\teng^{(m)}_{r_{m-1},i_m,r_m} = \bar u_{r_{m-1}+(I_m-1)R_{m-1},~r_m}^{(m)}$, \quad $\bar \bV_m = [\bar v_1^{(m)}, \bar v_2^{(m)}, \ldots, \bar v_{R_m}^{(m)}]$ \label{ttcore}
		\ENDFOR
		\STATE $\hat \bZ_{m+1} := \bS_m \bar \bV_m^T$
		\ENDFOR 
		\STATE $\teng^{(M)} = \hat \bZ_M$
	\end{algorithmic}
\end{algorithm}

\begin{lemma}\label{lem}
\kk{Assume that the singular values $\sigma_1^{(m)}, \ldots, \sigma_{R_m}^{(m)}$ are simple for each $m=1,\ldots,M-1$.
Then Algorithm~\ref{UETTsvd} produces the unique TT decomposition.}
\end{lemma}
\begin{proof}
The $m$-th TT core produced in TT-SVD is a reshape of the left singular vectors of the Gram matrix of the current unfolding,
$
\bA_m := \bZ_m \bZ_m^\top.
$
To set up an induction, we notice that $\bZ_1 = \hat \bZ_1 = \tensx_{(1)}$ is unique,
and assume that $\bZ_m$ is unique too.
Consider the eigenvalue decomposition
$\bA_m \bU_m = \bU_m \Lambda_m$,
$\Lambda_m = \text{diag}(\lambda_1^{(m)}, \ldots, \lambda_{R_{m-1}I_m}^{(m)})$.
Since the eigenvalues $\lambda_{i}^{(m)} = (\sigma_{i}^{(m)})^2$ are simple,
each of them corresponds to an eigenspace of dimension $1$, spanned by the corresponding column of $\bU_m$.
This means that each eigenvector is unique up to a scalar factor, and, if the eigenvector is real and has Euclidean norm, the scalar factor can only be $1$ or $-1$.
The latter is unique if we choose it as the sign of the largest in modulus element of the eigenvector (which is always nonzero), with ties broken to take the first of identical elements.
It remains to establish the uniqueness of $\bZ_{m+1}$ to complete the induction.
By the orthogonality of $\bar\bU_m=[\bar u_1^{(m)}, \ldots, \bar u_{R_m}^{(m)}]$,
we get $\hat \bZ_{m+1} = \bar\bU_{m}^T \bZ_{m}$, and since the reshape is unique, so is $\bZ_{m+1}$.
\end{proof}

%

\begin{remark}
\kk{Most of the data featuring in machine learning are noisy.
Therefore, the singular values of the corresponding matricizations are simple almost surely, and hence the TT decomposition delivered by Algorithm~\ref{UETTsvd} is unique almost surely.}
\end{remark}

\subsection{TT-CP Expansion}\label{TT-CP}
Despite the difficulties in \emph{computing} a CP approximation, its simplicity makes the CP format a convenient and powerful tool for revealing hidden classification features in the input data.
However, as long as the TT decomposition is available, it can be
converted into the CP format suitable for the kernelized classification.


\begin{proposition}\label{lem2}
\kk{For a given TT decomposition~\eqref{eq:tt}, we can write a CP decomposition}
\begin{equation}\label{eq:TT-CP}
\sum_{r_0,\ldots,r_{M}}\tendtt G 1_{r_0,i_1,r_1} \tendtt G 2_{r_1,i_2,r_2} \cdots \tendtt G M_{r_{M-1},i_M,r_M} =  \sum_{r=1}^{R} \hat H^{(1)}_{i_1,r} \hat H^{(2)}_{i_2,r} \cdots \hat H^{(M)}_{i_M,r}
\end{equation}
\kk{by merging the ranks $r_1, r_2, \ldots r_M$ into one index $r = r_1 + (r_2-1) R_1 + \ldots + (r_M-1) \prod_{\ell=1}^{M-1} R_{\ell}$, $r=1,\ldots,R$, $R=R_1\cdots R_M$, and introducing the CP factors}
$$
\kk{\hat H^{(m)}_{i_m,r}  = {\teng^{(m)}_{r_{m-1},i_m,r_m}}, \quad m = 1, \ldots, M.}$$
\end{proposition}

\kk{This transformation is free from any new computations, and needs simply rearranging and replicating the original TT cores. Although this expansion is valid for arbitrary dimension, higher dimensions may increase the number of terms massively.
However, many experimental datasets are usually three or four dimensional tensors, for which the TT-CP expansion is feasible.}

\kk{Note that the number of terms $R$ in the CP decomposition~\eqref{eq:TT-CP} can be larger than the minimal CP rank of the exact CP decomposition of the given tensor.
However, the nonlinear kernel function is more sensitive to the features of the data rather than the number of CP terms \emph{per se}.
In the numerical tests, we observe that the expansion~\eqref{eq:TT-CP} gives actually a better classification accuracy than an attempt to compute an optimal CP approximation using an ALS method.
}

\subsection{Norm Equilibration}\label{sec:normequ}
In our preliminary experiments, we tried using directly the TT-CP expansion as above with the CP kernel from~\citep{MMK}.
However, this did not lead to better classification results.
The DuSK kernel~\citep{MMK}
introduces the same width parameter for all CP factors.
This requires all CP factors to have identical (or at least close) magnitudes.
In contrast, different TT cores have different norms in the plain TT-SVD algorithm~\citep{oseledets2011tensor}.
Here, we rescale the TT-CP expansion
to ensure that the columns of the CP factors have equal norms,
and hence produce the same kernel features with the same width parameter.
We have found this to be a key ingredient for the successful TT-SVM classification.

Given a rank-$r$ TT-CP decomposition $\llbracket \hat H^{(1)}, \hat H^{(2)},\cdots, \hat H^{(M)}\rrbracket$,
we compute the total norm of each of the rank-1 tensors
\begin{equation}
\kk{n_r  = \norm{\hat H_{r}^{(1)}} \cdots \norm{\hat H_{r}^{(M)}},\label{norm1}}
\end{equation}
and distribute this norm equally among the factors,
\begin{equation}
\kk{H^{(m)}_{r} :=  \frac{\hat H_{r}^{(m)}}{\norm{\hat H_{r}^{(m)}}} \cdot n_r^{1/M}, \qquad m=1,2, \cdots, M.\label{norm2}}
\end{equation}


\subsection{Noise-aware Threshold and Rank Selection}
Generally, data coming from real world applications are affected by measurement or preprocessing noise.
This can affect both computational and modeling aspects,
increasing the TT ranks (since a tensor of noise lacks any meaningful TT decomposition), and spoiling the classification if the noise is too large.
However, the SVD can serve as a de-noising algorithm automatically:
the dominant singular vectors are often ``smooth'' and hence represent a useful signal, while the latter singular vectors
are more oscillating and capture primarily the noise.
Therefore, it is actually beneficial to compute the TT approximation with deliberately low TT ranks / large truncation threshold.
On the other hand, the TT rank must not be too low in order to approximate the features of the tensor with sufficient accuracy.
Cross-validation is a technique to evaluate the effectiveness of the model, which is done by re-sampling the data into training-testing data sets. Since the precise magnitude of the noise is unknown,
we carry out a k-fold cross-validation test (k = 5) to find the optimal TT rank. 

\subsection{Nonlinear Mapping}\label{kermapping}
Equipped with the homogenized TT-CP decompositions of the input tensors,
we are ready to define a nonlinear kernel function.
We follow closely the rationale behind DuSK proposed in~\citet{DuSK,MMK} \kk{and generalize it for tensors of arbitrary dimension}.
We assume that the feature map function \kk{from the space of tensors to a \emph{tensor product Reproducing Kernel Hilbert Space}~\citep{Signoretto2013} $\Psi  \colon \mathbb{R}^{I_1} \times \cdots \times \mathbb{R}^{I_M} \mapsto \mathbb{F}$} consists of separate feature maps acting on different CP factors,
\kk{
\begin{align}\label{eq:NM}
\Psi & \colon\hspace{-0.1cm}\sum_{r = 1}^{R} H_r^{(1)}\otimes H_r^{(2)}\otimes \cdots \otimes H_r^{(M)}\mapsto\sum_{r = 1}^{R} \phi(H_r^{(1)})\otimes \phi(H_r^{(2)})\otimes \cdots \otimes  \phi(H_r^{(M)}).
\end{align}}
This allows us to exploit the fact that the data is given in the CP format to aid the classification.
However, the feature function $\phi(\ba)$ is to be defined implicitly through a kernel function.
Similarly to the standard SVM, applying the kernel trick to~\eqref{eq:NM} gives us a practically computable kernel.
Given CP approximations of two tensors $\tensx = [x_{i_1,\ldots,i_M}]$ and $\tensy = [x_{i_1,\ldots,i_M}]$,
\begin{align*}
x_{i_1,\ldots,i_M} \approx \sum_{r=1}^{R} H^{(1)}_{i_1,r} H^{(2)}_{i_2,r} \cdots  H^{(M)}_{i_M,r}, \qquad
y_{i_1,\ldots,i_M} \approx \sum_{r=1}^{R} P^{(1)}_{i_1,r} P^{(2)}_{i_2,r} \cdots  P^{(M)}_{i_M,r},
\end{align*}
we compute
\kk{
\begin{align}\label{eq:kerAPRX}
\langle \Psi(\tensx), \Psi(\tensy) \rangle & =  K(\tensx,\tensy)\nonumber\\
& = K \left( \sum_{r=1}^{R} H_r^{(1)}\otimes H_r^{(2)}\otimes \cdots \otimes H_r^{(M)}, \sum_{r=1}^{R} P_r^{(1)}\otimes P_r^{(2)}\otimes \cdots \otimes P_r^{(M)}  \right),\nonumber\\
& = \langle{\Psi(\sum_{r=1}^{R} H_r^{(1)}\otimes H_r^{(2)}\otimes \cdots \otimes H_r^{(M)}) ,\Psi(\sum_{r=1}^{R} P_r^{(1)}\otimes P_r^{(2)}\otimes \cdots \otimes  P_r^{(M)})}\rangle \nonumber\\
& = \sum_{i,j=1}^{R}\langle\phi(H_{i}^{(1)}),\phi(P_{j}^{(1)})\rangle \langle\phi(H_{i}^{(2)}),\phi(P_{j}^{(2)})\rangle \cdots  \langle\phi(H_{i}^{(M)}),\phi(P_{j}^{(M)})\rangle \nonumber\\
& = \sum_{i,j = 1}^{R}\nk(H_{i}^{(1)},P_{j}^{(1)}) \nk(H_{i}^{(2)},P_{j}^{(2)}) \cdots \nk(H_{i}^{(M)},P_{j}^{(M)}),
\end{align}}
where \begin{center}
 $ \quad \nk(\bh,\bp)  =  \exp\left(-\frac{\norm{{\bh - \bp}}^2}{2 \sigma^2}\right).$
\end{center}

This kernel approximation is computed for each pair of the tensor input data, represented by its CP factors.
The width parameter $\sigma>0$ needs to be chosen judiciously to ensure accurate learning.

Since the entire calculation starts from the TT decomposition,
we call this proposed model the \emph{Tensor Train Multi-way Multi-level Kernel (TT-MMK)}.
It fulfills the objectives of extracting optimal low-rank features, and of building a more accurate and efficient classification model.
Plugging the kernel values~\eqref{eq:kerAPRX} into the STM optimizer \eqref{eq: kerSTM} completes the algorithm.
The overall idea is summarized in Algorithm~\ref{alg:KTTCP}.

\begin{algorithm}[ht]
	\caption{TT-CP approximation of the STM Kernel}
	\label{alg:KTTCP}
	\begin{algorithmic}
		\STATE {\bfseries Input:} data ${\{\tensx_n\}}_{n = 1}^{N}  \dims R I M$, TT-rank $R$. 
		\STATE {\bfseries Output:} Kernel matrix approximation $\left[K(\tensx_u,\tensx_v)\right] \in \R^{N\times N}$
		\FOR{$n = 1$ {\bfseries to} $N$}
		\STATE Compute TT approximation $\tensx_n  \cong \llangle \teng^{(1,n)}, \teng^{(2,n)}, \cdots,  \teng^{(M,n)} \rrangle$ using Algorithm~\ref{UETTsvd}.
		\STATE Compute TT-CP expansion $\llbracket H^{(1,n)}, H^{(2,n)}, \cdots, H^{(M,n)}\rrbracket = \llangle \teng^{(1,n)}, \teng^{(2,n)},\cdots,  \teng^{(M,n)} \rrangle$ with equilibrated norms.
		\ENDFOR
		\FOR{$u, v = 1$ {\bfseries to} $N$}
		\STATE $K\left(\tensx_u,\tensx_v\right) \approx$ 
		$\sum_{i,j = 1}^{R}\nk(H_{i}^{(1,u)},H_{j}^{(1,v)}) \nk(H_{i}^{(2,u)},H_{j}^{(2,v)}) \cdots \nk(H_{i}^{(M,u)},H_{j}^{(M,v)}).$
		\ENDFOR	
	\end{algorithmic}
\end{algorithm}

\section{Numerical Tests}\label{numel}
\begin{itemize}\setlength\itemsep{0.5em}
	\item \textbf{Experimental Settings}\\
	All numerical experiments have been done in \texttt{MATLAB 2016b}. In the first step, we compute the TT format of an input tensor using the \texttt{TT-Toolbox\footnote{\url{https://github.com/oseledets/TT-Toolbox}}}, where we modified the function \verb+@tt_tensor/round.m+ to enforce the uniqueness enforcing TT-SVD (\secref{sec: USVD}).
	Moreover, we have implemented the TT-CP conversion, together with the norm equilibration.
	For the training of the TT-MMK model, we have used the \texttt{svmtrain} function available in the \texttt{LIBSVM\footnote{\url{https://www.csie.ntu.edu.tw/~cjlin/libsvm/}}} library.
	We have run all experiments on a machine equipped with Ubuntu release 16.04.6 LTS 64-bit, 7.7 GiB of memory, and an Intel Core i5-6600 CPU @ 3.30GHz$\times$4 CPU. The codes are available publicly on \texttt{GitHub\footnote{\url{https://github.com/mpimd-csc/Structure-preserving_STTM}}}.
	\item \textbf{Parameter Tuning}\\
	The entire TT-SVM model depends on three parameters. First, to simplify the selection of TT ranks, we take all TT ranks equal to the same value $R \in \lbrace 1,2, \ldots 10 \rbrace$. Another parameter is the width of the Gaussian Kernel $\sigma$.
	Finally, the third parameter is a trade-off constant $C$ for the KSTM optimization technique \eqref{eq: kerSTM}. Both $\sigma$ and $C$ are chosen from $\lbrace 2^{-8},2^{-7}, \ldots, 2^{7}, 2^{8} \rbrace$. For tuning $R,\sigma$ and $C$ to the best classification accuracy, we use the \emph{$k$-fold cross validation} with $k=5$. Along with this, we repeat all computations  50 times and average the accuracy over these runs.
	This ensures a confident and reproducible comparison of different techniques.
\end{itemize}
\subsection{Data Collection}
\begin{enumerate}\setlength\itemsep{0.5em}
	\item \kk{Resting-state fMRI Datasets}
	\begin{itemize}
		\item \textbf{Alzheimer Disease (ADNI):} The ADNI\footnote{\url{http://adni.loni.usc.edu/}} stands for Alzheimer Disease Neuroimaging Initiative. It contains the resting state fMRI images of 33 subjects. The data set was collected from the authors of the paper~\citep{MMK}. The images belong to either Mild Cognitive Impairment (MCI) with Alzheimer Disease (AD), or normal controls. Each image is a tensor of size $61 \times 73 \times 61$, containing $271633$ elements in total. The AD+MCI images are labeled with $-1$, and the normal control images are labeled with $1$. Preprocessing of the data sets is explained in~\citep{DuSK}.
		
		\item \textbf{Attention Deficit Hyperactivity Disorder (ADHD):} The ADHD data set is collected from the ADHD-200 global competition data set\footnote{\url{http://neurobureau.projects.nitrc.org/ADHD200/Data.html}}. It is a publicly available preprocessed fMRI data set from eight different institutes, collected at one place. The original data set is unbalanced, so we have chosen 200 subjects randomly, ensuring that 100 of them are ADHD patients (assigned the classification label $-1$) and the 100 other subjects are healthy (denoted with label $1$). Each of the 200 resting state fMRI samples contains $49  \times 58 \times 47  = 133574$ voxels.
	\end{itemize}
	\item \kk{Hyperspectral Image (HSI) Datasets: 
	We have taken the} \verb+mat+ \kk{file for both the datasets and their corresponding labels\footnote{\url{http://www.ehu.eus/ccwintco/index.php/Hyperspectral\_Remote\_Sensing\_Scenes}}. The following datasets have three dimensional tensor structure of different sizes, where each tensor data point represents a pixel value. Therefore, for our experiment we have taken a patch of size $5 \times 5$ for two different pixel values, in order to get a binary classification dataset. }
	\begin{itemize}
		\item \kk{\textbf{Indian Pines:} The HSI images were collected via the Aviris Sensor\footnote{\url{https://aviris.jpl.nasa.gov/}} over the Indian Pines test site. The size of the dataset is $145 \times 145$ pixels over 224 spectral values. Hence, the size of the tensor data is $145 \times 145 \times 224$. The mat file we have collected for our experiment has reduced band size 200. This excludes bands covering the region of water absorption: [104-108], [150-163]. The original dataset contains 16 different labels to identify different corps and living areas. We have taken only 50 datapoints for each of the two labels 11 (Soybean-mintill ) and 7 (Grass-pasture-mowed). }
		
		\item \kk{\textbf{Salinas:} This HSI images data was collected by 224 band Aviris Sensor over Salinas valley, California. Similar to Indian Pines, in this case, we have also collected samples  for two GroundTruths, namely 9 (Soil-vinyard-develop) and 15 (Vinyard-untrained) each with 50 datapoints. The size of the dataset is $512 \times 217$ pixels over 224 spectral values. Hence, the size of the tensor data is $512 \times 217 \times 224$. }
	\end{itemize}
\end{enumerate}

\begin{figure}[ht]
	\begin{center}
		\centering
		\subfloat[\kk{GroundTruth of Indian Pines dataset} ]{\label{fig:hsispatial}\includegraphics[width=.5\linewidth]{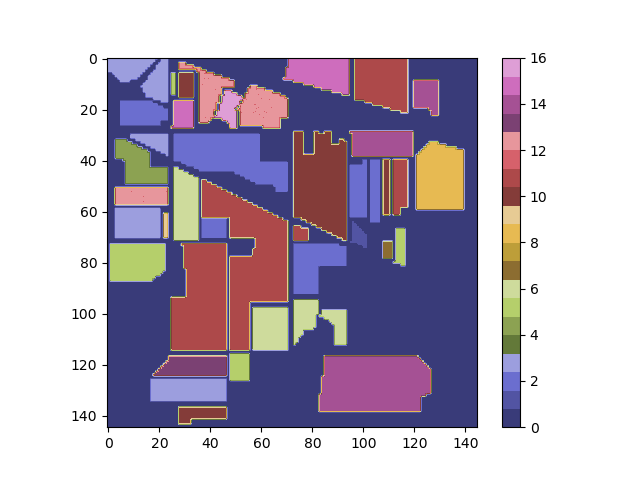}}
		\subfloat[\kk{GroundTruth of Salinas dataset}]{\label{fig:fmrispatial}	
		\includegraphics[scale = 0.45]{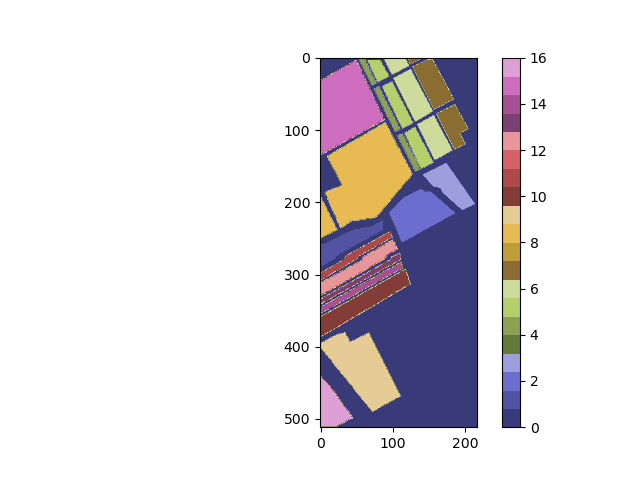}}
	\end{center}   
	\vskip -0.2in
\end{figure}
\subsection{Influence of Individual Algorithmic Steps}\label{methodology}
In the first test we investigate the impact of each individual transformation of the TT decomposition, outlined in~\secref{sec: USVD}--\secref{sec:normequ}.
In~\figref{fig:a}, we apply a counterpart of the DuSK kernel~\eqref{eq:kerAPRX}
directly to the initial TT approximation of the data tensors.
Given TT decompositions
$$
x_{i_1,i_2,i_3} = \sum_{r_1,r_2=1}^{R_1,R_2} \teng^{(1)}_{i_1,r_1} \teng^{(2)}_{r_1,i_2,r_2} \teng^{(3)}_{r_2,i_3} \quad\mbox{and}\quad
y_{i_1,i_2,i_3} = \sum_{t_1,t_2=1}^{R_1,R_2} \tens^{(1)}_{i_1,t_1} \tens^{(2)}_{t_1,i_2,t_2} \tens^{(3)}_{t_2,i_3},
$$
we compute a separable kernel similarly to~\eqref{eq:kerAPRX} via
\begin{equation}\label{eq:ttdusk}
k(\tensx, \tensy) = \sum_{r_1,t_1=1}^{R_1} \sum_{r_2,t_2=1}^{R_2} \nk(\teng^{(1)}_{r_1}, \tens^{(1)}_{t_1}) \nk(\teng^{(2)}_{r_1,r_2}, \tens^{(2)}_{t_1,t_2}) \nk(\teng^{(3)}_{r_2}, \tens^{(3)}_{t_2}).
\end{equation}
\kk{A similar approach was also proposed recently in~\citet{chen2020kernelized}.
In~\figref{fig:a} we have compared the difference between applying the DuSK kernel with TT decomposition and uniqueness enforced TT-SVD decomposition. However, we observe a rather poor classification accuracy both with and without uniqueness enforcing TT-SVD  (\secref{sec: USVD}). }


\kk{Applying TT-UoSVD DuSK was not sufficient to produce the state-of-the-art results. Therefore, in~\figref{fig:b} we convert the TT decomposition into the CP format as per~\secref{TT-CP}.
This makes a more significant difference, in particular it slightly increases the accuracy and its distribution is much more uniform across the TT ranks.
In~\figref{fig:c}, both  TT and TTCP with uniqueness enforced constraint  indicates that although UoSVD helps to stabilize the decomposition, the 'naive' TT DuSK~\eqref{eq:ttdusk} suffers from inhomogeneous contribution of the TT cores of different size because DuSK relies essentially on the number of univariate kernel terms being equal across different variables.}

\kk{To equilibrate the scales of different terms even further,
in~\figref{fig:d} we enable the norm distribution as shown in~\secref{sec:normequ}. This gives the highest classification accuracy,
which relies on the internal features of the images instead of the scale indeterminacy of their representations. The comparison in~\figref{fig:d} makes the impact of UoSVD also dominant to achieve the state-of-the-art.
We compare our algorithm to other methods next. }

 \begin{figure}[ht]
 	\centering
 	\subfloat[TT with and without enforced uniqueness ]{\label{fig:a}\includegraphics[width=.5\linewidth]{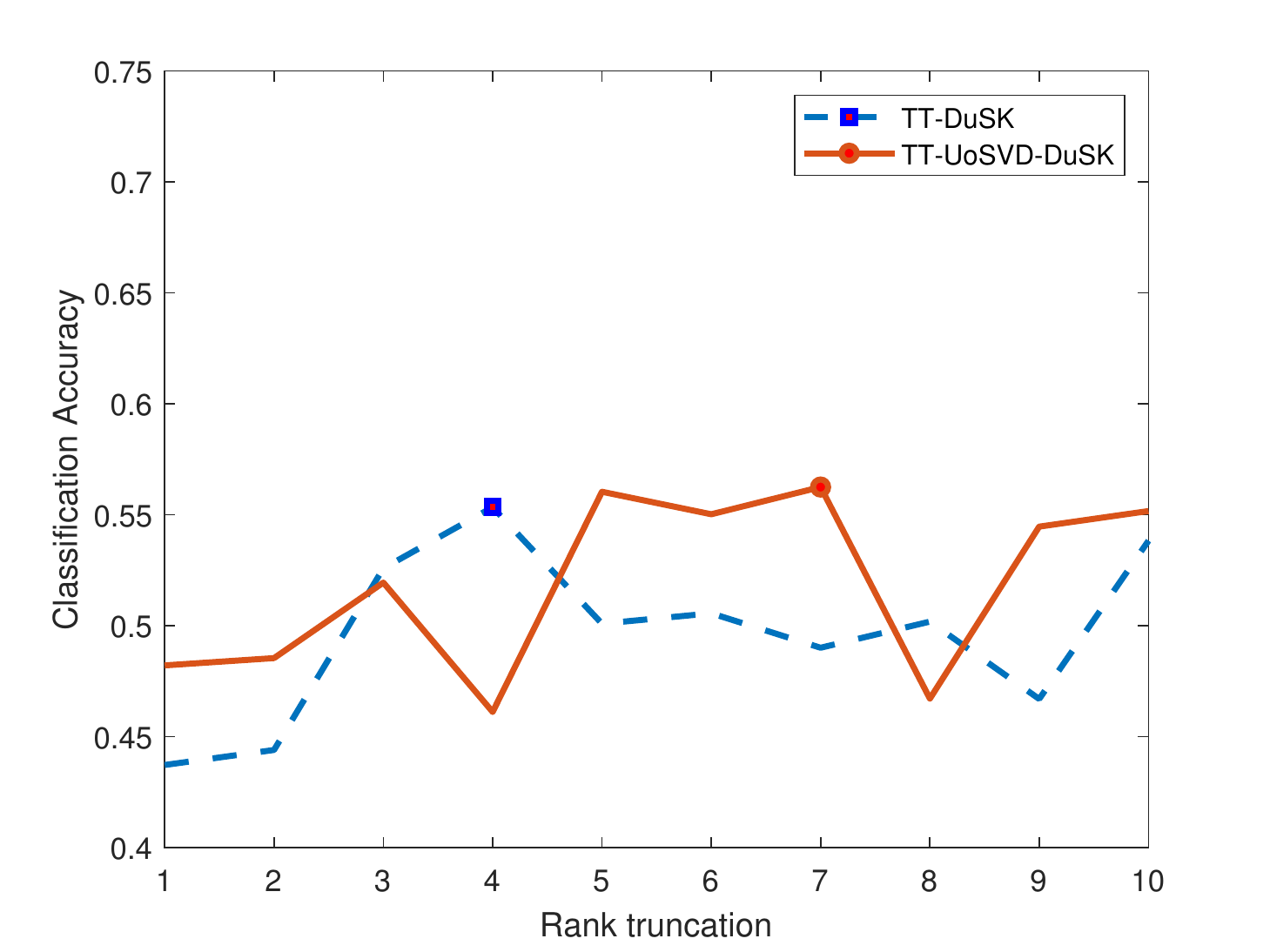}}
 	\subfloat[TTCP with and without enforced uniqueness]{\label{fig:b}	\includegraphics[width=.5\linewidth]{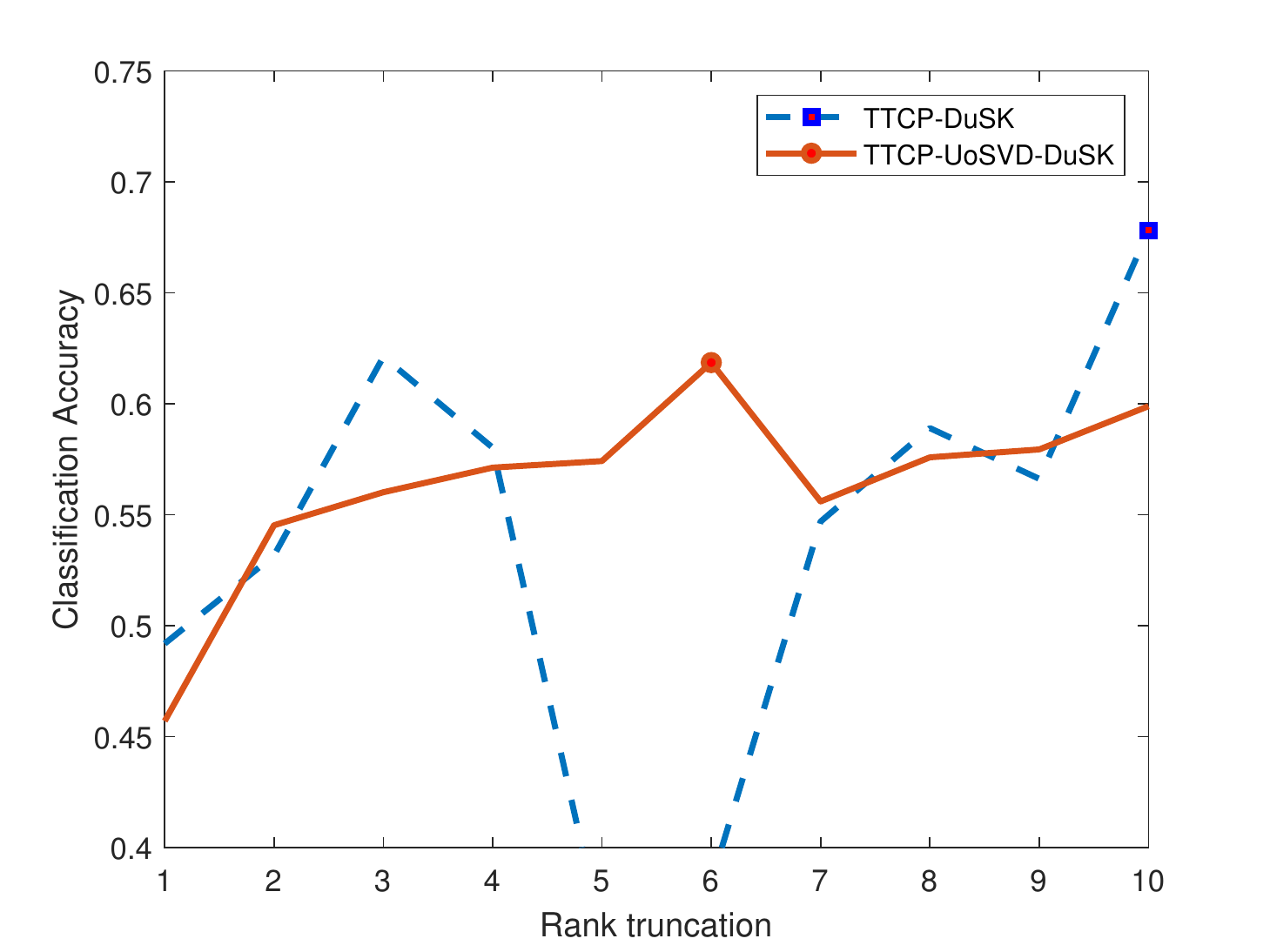}}
 	\newline
 	\subfloat[ TT vs TTCP with enforced uniqueness]{\label{fig:c}	\includegraphics[width=.5\linewidth]{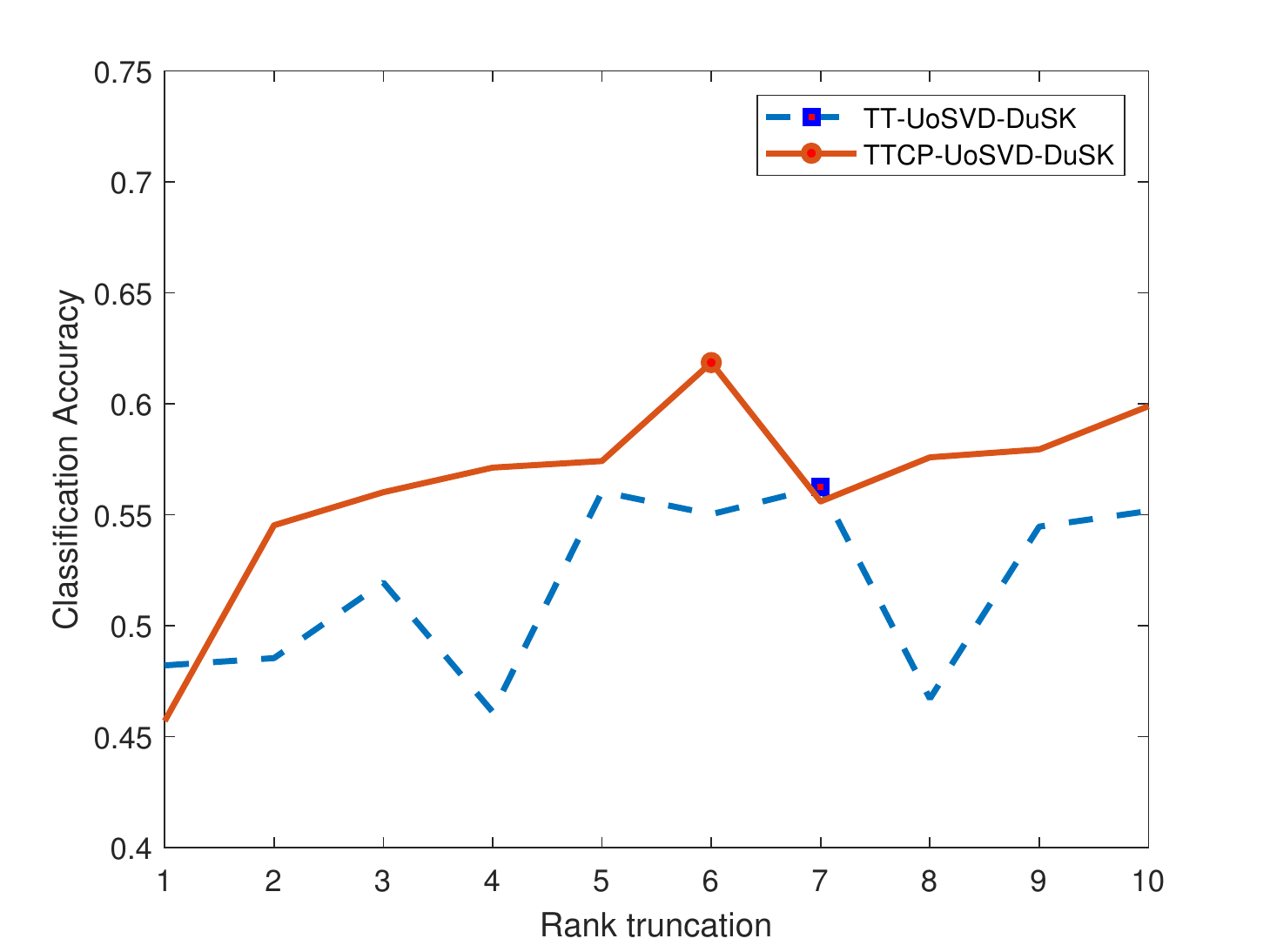}}
 	\subfloat[TTCP norm equilibrium with and without eforced uniqueness]{\label{fig:d}	\includegraphics[width=.5\linewidth]{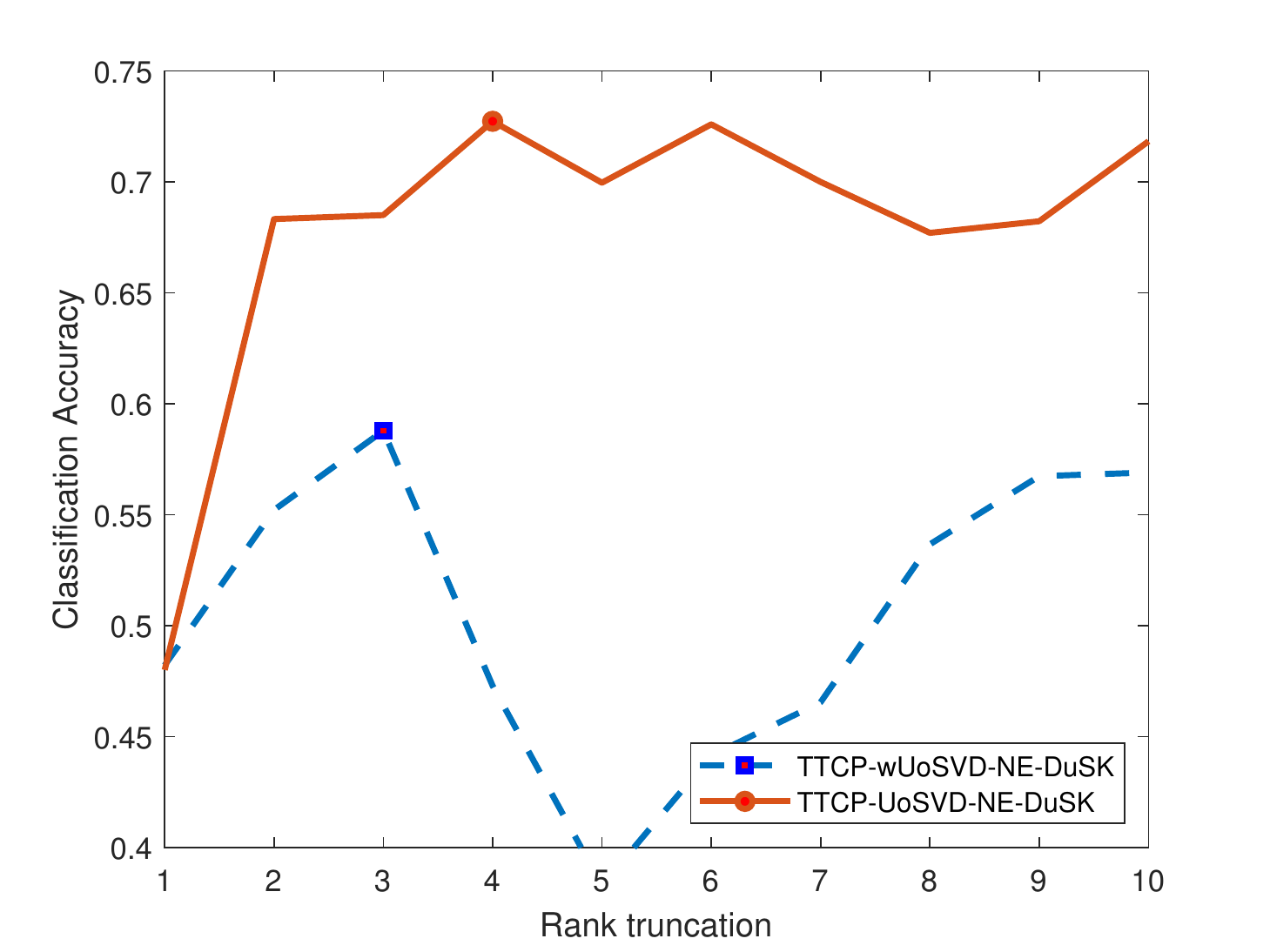}}
 	\caption{\kk{Effect of each step of~\secref{overalgo} on the classification accuracy for ADNI data set.}}
 	\label{fig:sec3compare}
 \end{figure}
\subsection{Comparison to Other Methods}
Next, we compare the classification accuracy of the whole proposed TT-MMK method (as per~\figref{fig:d})
with the accuracy of the following existing approaches.
\begin{itemize}\setlength\itemsep{0.1em} 
	\item[] \textbf{SVM:} the standard SVM with  Gaussian Kernel. This is the most used optimization method for vector input based on the maximum margin technique.
	\item[] \textbf{STuM:} The Support Tucker Machine (STuM)~\citep{kotsia} uses the Tucker decomposition. The weight parameters of the SVM are computed for optimization into Tucker factorization form.
	\item[] \textbf{DuSK:} The idea of DuSK~\citep{DuSK} is based on defining the kernel approximation for the rank-one decomposition. This is one of the first methods in this direction.
	\item[] \textbf{MMK:} This method is an extension of DuSK to the KCP input. It uses the covariance/random matrix projection over the CP factor matrices, to get a KCP for the given input tensor~\citep{MMK}.
	We used the original DuSK and MMK codes provided by the authors of the paper~\citep{MMK}.
	\item[] \textbf{Improved MMK:} This is actually a simplified MMK, where the projection of the CP onto the KCP is omitted (the covariance/random matrices are replaced by the identity matrices).
	\item[] \textbf{TT-MMK:} This is Algorithm~\ref{alg:KTTCP} proposed in this paper.
\end{itemize}

The results are shown in~\figref{kttcp} and Table~\ref{notation-table}.
Our key observations are as follows.
\begin{itemize}\setlength\itemsep{0.1em}
	\item[]\textbf{(In)sensitivity to the TT Rank Selection:}~\figref{kttcp}  shows that the proposed method gives almost the same accuracy for different TT ranks.
	For some samples, even the TT rank of $2$ gives a good classification.
	Note that this is not the case for MMK, which requires a careful selection of the CP rank.
	\item[]\textbf{Computational Robustness:} while the CP decomposition can be computed using only iterative methods in general, all steps of the kernel computation in TT-MMK are ``direct'' in a sense that they require a fixed number of linear algebra operations, such as the SVD and matrix products.
	\item[]\textbf{Computational Complexity:} approximating the full tensor in the TT format has the same $\cO(n^{M+1})$ cost as the Tucker and CP decompositions.
	All further operations with factors scale linearly in the dimension $M$ and mode sizes, and polynomially in the ranks.
	\item[]\textbf{Classification Accuracy:} the proposed method gives the best average classification accuracy
        compared to five other state of the art techniques.
	\item[]\kk{\textbf{Generalization:} Top accuracy in datasets from two different areas (fMRI and HSI) shows that the method is suitable for a wide range of binary tensor classification problems.}
\end{itemize}
\begin{table}[ht]
	\caption{\kk{Average classification accuracy in percentage for different methods and data sets}}
	\label{notation-table}
	\vskip 0.15in
	\begin{center}
		\begin{small}
			\begin{sc}
				\begin{tabular}{lccccr}
					\toprule
					Methods & ADNI & ADHD & Indian Pines & Salinas \\
					\midrule
					SVM &  49  &  52 & 46 & 47\\
					STuM & 51  & 54 & 57&74\\
					DuSK & 55 & 58  & 60& 92\\
					MMK& 69  &  60 & 93& 98 \\
					Improved  MMK & 71  &61  & 94 &98\\
					TT-MMK &  \textbf{73} &  \textbf{64} & \textbf{99} & \textbf{99}\\
					\bottomrule
				\end{tabular}
			\end{sc}
		\end{small}
	\end{center}
	\vskip -0.1in
\end{table}
\begin{figure}[ht]
	\begin{center}
		\centerline{\includegraphics[width=0.5\textwidth, height = 0.4\textwidth]{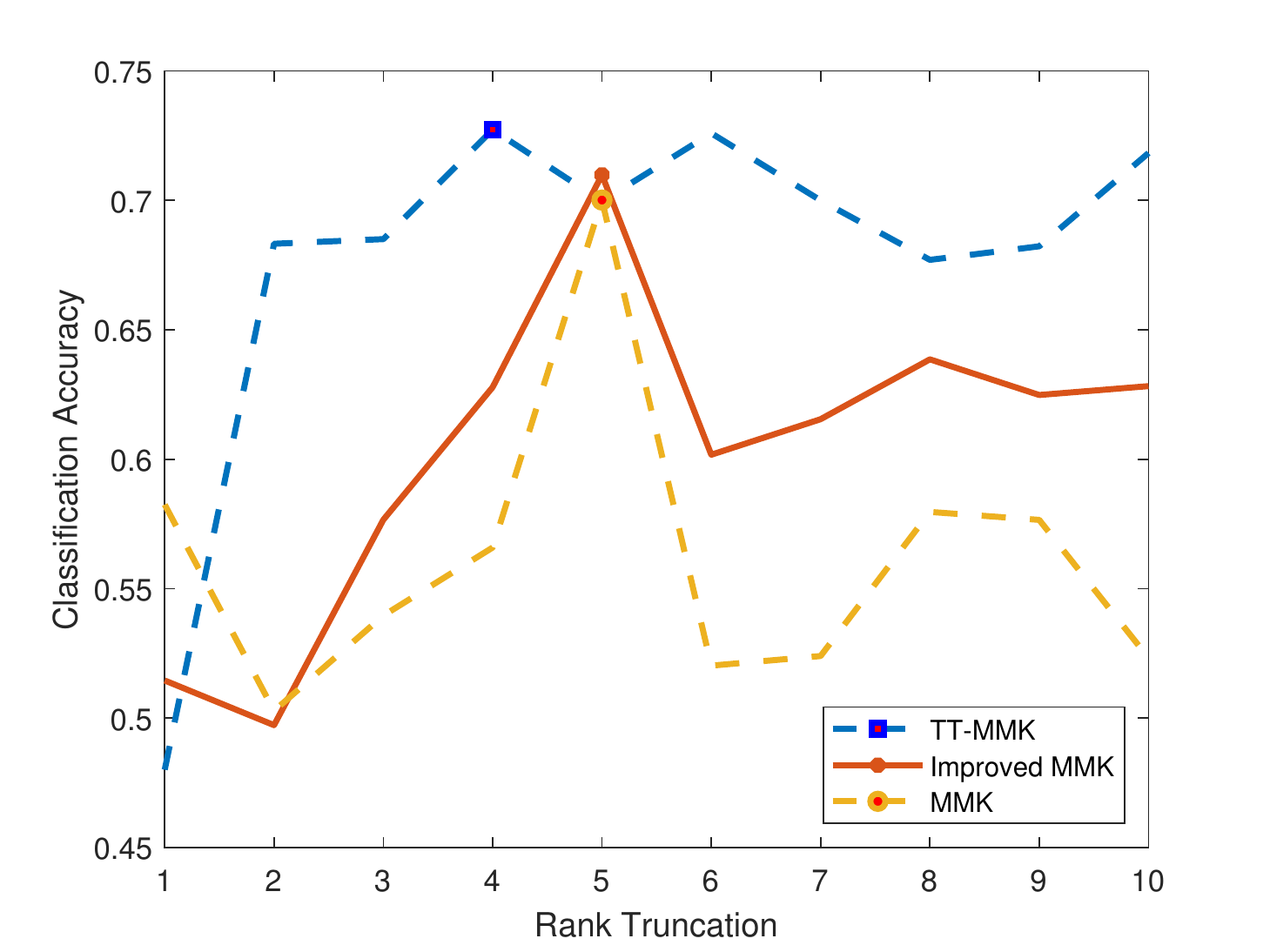}}
		\caption{Accuracy of TT-MMK (TTCP), MMK and Improved MMK methods for the ADNI data set using different TT/CP ranks respectively. TT-MMK and Improved MMK methods use identity kernel filtering while the MMK uses random kernel filtering.}
		\label{kttcp}
	\end{center}
	\vskip -0.2in
\end{figure}
\section{Conclusions}
	
	We have proposed a new kernel model for SVM classification of tensor input data. Our kernel extends the DuSK approach~\citep{MMK} to the TT decomposition of the input tensor
with enforced uniqueness and norm distribution. The TT decomposition can be computed more reliably than the CP decomposition used in the original DuSK kernel.
Using fMRI \kk{and Hyperspectral Image} data sets, we have demonstrated that the new TT-MMK method provides higher classification accuracy for an unsophisticated choice of the
TT ranks
for a wide range of classification problems.
We have found out that the each component of the proposed scheme (uniqueness \kk{enforced TT, TT-CP expansion and norm equilibration}) is crucial for achieving this accuracy.

Further research will consider improving the computational
complexity of the current scheme, as well as a joint optimization of the TT cores and SVM weights. Similarly to the neural network compression in the TT format~\citep{novikov-TT-NN-2015}, such a targeted
iterative refinement of the TT decomposition may improve the prediction accuracy.


\acks{This work is supported by the International Max Planck Research School for Advanced Methods in Process and System Engineering-\href{https://www.mpi-magdeburg.mpg.de/imprs}{\textbf{IMPRS ProEng}}, Magdeburg and is a part of Max Planck research network on Big Data Driven Material Science (\href{https://www.bigmax.mpg.de/}{\textbf{BiGmax}}) project. The authors are grateful to Lifang He for providing codes for the purpose of comparison.}

\vskip 0.2in
\bibliography{JMLR_Eff_STTM}

\end{document}